\documentclass{article}

\usepackage[colorinlistoftodos, textwidth=26mm, shadow,color=blue!30!white, disable]{todonotes}

\newcommand{\todoy}[2][]{\todo[color=red!20!white,#1]{Yasin: #2}}

\usepackage{algorithm}
\usepackage{algorithmic}

\usepackage{amsmath,abbrevs}
\usepackage{amssymb,amsthm}

\usepackage{color}
\usepackage{graphicx}
\usepackage{epstopdf}
\usepackage{algorithm,algorithmic}
\usepackage{verbatim}

\usepackage{nicefrac}

\newif\ifcomm
\commtrue

\newif\iflong
\longtrue

\newtheorem{thm}{Theorem}

\newcounter{assumption}
\renewcommand{\theassumption}{A\arabic{assumption}}

\newcommand{\Real}{\mathbb R}                        

\newcommand{\ra}{\rightarrow}

\newcommand{\argmin}{\mathop{\rm argmin}}

\newcommand{\beq}{\begin{equation}}
\newcommand{\eeq}{\end{equation}}

\ifcomm
   \newcommand\comm[1]{\textcolor{blue}{ #1}}
   \newcommand{\mtodo}[2]{\todo{{\bf #1}: #2}} 
   \def\here#1{{\bf $\langle\langle$#1$\rangle\rangle$}}

\else
   \newcommand\comm[1]{}
   \newcommand{\mtodo}[2]{}
   \def\here#1{}
\fi

\newcommand{\cG}{{\cal G}}

\newcommand{\cH}{{\cal H}}
\renewcommand{\phi}{\varphi}

\newcommand{\cN}{{\cal N}}
\newcommand{\cL}{{\cal L}}

\usepackage{subfig}
\usepackage{float}

\usepackage[round,comma]{natbib}
\usepackage{microtype}
\usepackage{graphicx}
\usepackage{booktabs} 

\usepackage{hyperref}

\usepackage[accepted]{icml2018}

\icmltitlerunning{A Continuation Method for Discrete Optimization}

\begin{document}

\twocolumn[
\icmltitle{A Continuation Method for Discrete Optimization and its Application\\ to Nearest Neighbor Classification}
           

\begin{icmlauthorlist}
\icmlauthor{Ali Shameli}{to}
\icmlauthor{Yasin Abbasi-Yadkori}{goo}
\end{icmlauthorlist}

\icmlaffiliation{to}{Department of Management Science and Engineering, Stanford University, Stanford, CA, USA}
\icmlaffiliation{goo}{Adobe Research, San Jose, CA, USA}

\icmlcorrespondingauthor{Ali Shameli}{shameli@stanford.edu}


\vskip 0.3in
]



\printAffiliationsAndNotice{}  


\begin{abstract}
The continuation method is a popular approach in non-convex optimization and computer vision. The main idea is to start from a simple function that can be minimized efficiently, and gradually transform it to the more complicated original objective function. The solution of the simpler problem is used as the starting point to solve the original problem. We show a continuation method for discrete optimization problems. Ideally, we would like the evolved function to be \emph{hill-climbing friendly} and to have the same global minima as the original function. We show that the proposed continuation method is the best affine approximation of a transformation that is guaranteed to transform the function to a hill-climbing friendly function and to have the same global minima.  

We show the effectiveness of the proposed technique in the problem of nearest neighbor classification. 
Although nearest neighbor methods are often competitive in terms of sample efficiency, the computational complexity in the test phase has been a major obstacle in their applicability in big data problems. Using the proposed continuation method, we show an improved graph-based nearest neighbor algorithm. 
The method is readily understood and easy to implement. We show how the computational complexity of the method in the test phase scales gracefully with the size of the training set, a property that is particularly important in big data applications. 
\end{abstract}

\section{Introduction}

The continuation method is a popular approach in non-convex optimization and  computer vision~\citep{Witkin-Terzopoulos-Kass-1987, Terzopoulos-1988, Blake-Zisserman-1987, Yuille-1987, Yuille-1990, Yuille-Peterson-Honda-1991}. The main idea is to start from a simple objective function that can be minimized efficiently, and gradually transform it to the more complicated original objective function. The method produces a sequence of optimization problems that are progressively more complicated and closer to the original nonconvex problem. The solution of a simpler problem is used as the starting point to solve a more complicated problem.
Given an appropriate transformation, we expect that the solution of the simpler problem is close to the solution of the more complicated problem. 

In this work, we show a continuation method for discrete optimization problems. We define a notion of regularity. We say that function $f$ is hill-climbing friendly (HCF) with respect to graph $\cG$ if by starting at any node on the graph, and moving along the edges using hill-climbing, we will eventually stop at the global minimum of $f$. 
\todoy{does this have a name?!} Ideally, we would like the evolved function to be HCF and have the same global minima as the original function. Then minimizing the evolved function also solves the original optimization problem. We show a continuation method that satisfies these properties. Unfortunately, computing the corresponding transformations is computationally expensive. We show that the best affine approximation of this continuation method can be computed efficiently. 
This result parallels the recent work of \citet{mobahi2015link} in continuous domains.  

Next, we describe the intuition behind the continuation method. Given a graph $\cG$ with the set of nodes $S$, and a function $f:S \ra \Real$, we are interested in finding a point in $S$ that minimizes $f$. In the ideal case and when the function $f$ is HCF, a hill-climbing method finds the global minima. Unfortunately, in many real-world applications, function $f$ is not HCF. Next, we show a transformation of $f$ to a HCF function. Let $N_i\subset S$ be the set of neighbors of node $i\in S$ in $\cG$ and $c\in\Real$ be a non-negative constant. Consider a sequence of functions defined by, $\forall i\in S$, 
\begin{align*}
f(i, 0)&=f(i)\,, \\
f(i, t+1)&=\min( f(i, t),  h(i,t) ) \,,
\end{align*}
where 
\[
h(i,t) = \min_{u\in N_i} f(u,t) + c \;.
\]
Function $f(i, t)$ represents the value of node $i$ at smoothing round $t$. Function $f(i,t)$ is ``more HCF'' for larger $t$. Let $D$ be the diameter of the graph. We can see that for small enough $c$, the above sequence will eventually produce a convex function after no more than $D$ steps and also the the optimum point will not change after this transformation. Let $t$ be large enough such that $f(.,t)$ is HCF. Given that $f(.,t)$ is HCF and $\argmin_{i\in S} f(i) \in \argmin_{i\in S} f(i,t)$, we only need to find the minima of $f(.,t)$. Unfortunately, computing  $f(.,t)$ can be computationally expensive. We show that an approximation of $f(.,t)$ however can be computed efficiently. In fact, it turns out that performing random walks provides an approximation that is optimal in some sense: Let $\cN$ be the operator defining the above sequence; $f(i,t+1) = \cN\{f(.,t)\}(i)$. The best affine approximation of $\cN$ can be computed by performing a random walk of length 1 on the graph. 

We apply the continuation method to a class of nearest neighbor search problems that are defined as optimization problems on proximity graphs. Before explaining the method, we discuss the motivation behind studying this particular class of nearest neighbor algorithms. Nearest neighbor methods are among the oldest solutions to classification and regression problems.  
Given a set of points $S=\{x_1, x_2, \ldots, x_n\}\subset \Real^d$, a query point $y\in\Real^d$, and a distance function $d:\Real^d\times \Real^d \ra \Real$, we are interested in solving 
\beq
\label{eq:opt}
\argmin_{x\in S} d(x,y) \;.
\eeq
In a supervised learning problem, $S$ is the training set, and $y$ is a test point. The algorithm uses the label of the solution of the above problem to predict the label of $y$. The bias-variance tradeoff can be balanced by finding a number of nearest neighbors and making a prediction according to these neighbors. 
A trivial solution to problem~\eqref{eq:opt} is to examine all points in $S$ and return the point with the smallest $f$ value. The computational complexity of this method is $O(n)$, and is not practical in large-scale problems. We are interested in an approximate algorithm with sublinear time complexity. 

When size of dataset is large, the global structure of the objective $f$ defined on data $S$ can be effectively inferred from local information; for any point in $S$, there will be many nearby points as measured in a simple distance metric such as Euclidean distance. Although the Euclidean distance is not the ideal choice in high dimensional problems, it often provides reasonable results when points are close to each other. If nearest neighbors of each point in $S$ were known, we could simply start from a random point in $S$ and greedily move in the direction that decreases the objective $f$. In an ideal case and when $S$ forms an $\epsilon$-cover of the space for a small $\epsilon$, this greedy procedure will most likely find the global minima of problem~\ref{eq:opt}. 
\todoy{say that distance function is convex and we expect this to work in the continuous space. Big data problems are like continuous problems...}

The above discussion motivates constructing a proximity graph in the training phase, and performing a hill-climbing search in the test phase. Let $N$ be a positive integer. Let $\cG$ be a proximity graph constructed on $S$ in an offline phase, i.e. set $S$ is the set of nodes of $\cG$, and each point in $S$ is connected to its $N$ nearest neighbors with respect to some distance metric. 
Given $\cG$, we solve problem~\eqref{eq:opt} by starting from an arbitrary node  and moving in the direction that decreases the objective $f$. The idea of performing hill-climbing on a proximity graph dates back at least to~\citet{Arya-Mount-1993}, and it has been further improved by \citet{Brito-1997, Eppstein-1997, Miller-Teng-1997, Plaku-Kavraki-2007, Chen-Fang-Saad-2010, Connor-Kumar-2010, Dong-Charikar-Li-2011, Hajebi-2011, Wang-Wang-2012}. 

The greedy procedure for solving problem~\eqref{eq:opt} can get stuck in local minima.  We apply our proposed continuation technique to solve problem~\eqref{eq:opt} on a proximity graph. The continuation method is less likely to get stuck in local minima, and demonstrates improved performance in practice. The resulting algorithm is a greedy method with added random walks that smoothen the optimization function. 
 \todoy{give a high level explanation of the smoothing method}
 The exact algorithm is described in Section~\ref{sec:method}. 
We show advantages of the proposed technique in two image classification datasets. The primary objective of our experiments is to show how the proposed method can take advantage of larger training datasets. Our current performance results are not competitive with state-of-the-art deep neural networks, although we believe further study and engineering of the proposed architecture might close this gap. 


Deep neural networks have shown effectiveness in many machine learning problems. In many successful deep architectures, the number of parameters is about the same or even exceeds the size of training dataset. These overparameterized networks are often trained to achieve a very low training error in a time consuming training phase~\citep{Salakhutdinov-2017,Ma-Bassily-Belkin-2017,Zhang-Bengio-2017}. Interestingly, such networks often provide low test error. 


Certain properties of deep learning methods, such as using large models and having small training error, bring to mind nonparametric techniques such as nearest neighbor and kernel methods that store the whole dataset.
Although nonparametric methods are often competitive in terms of sample efficiency, their computational complexity in the test phase has been a major obstacle in their applicability in big data problems. 
Another limitation of nonparametric methods is their reliance on an appropriate distance metric; the choice of distance metric can have a significant impact in the performance of the method. In contrast, deep learning methods are shown to be able to learn appropriate feature embeddings in a variety of problems. Further, deep learning methods are fast in the test phase, although the training phase can be computationally demanding.  
There have been a number of attempts to scale up kernel methods and make them competitive with deep neural nets~\citep{Williams-Seeger-2001, Rahimi-Recht-2008, Le-Sarlos-Smola-2013, Bach-2014, Lu-Sha-2014, Lu-Sha-2016}. 
Despite these efforts, there still appears to be a gap in the performance of nonparametric methods and deep learning techniques. Our approach can be viewed as a step towards closing this gap. 

The resulting nearest neighbor classification method has a number of features. First, the method is intuitive and easy to implement. Second, and most importantly, the computational complexity in the test phase scales gracefully with the size of dataset. Finally, in big data problems, the method can work well with simple metrics such as Euclidean distance. 

The continuation method can also be applied to other discrete optimization problems that can be formulated as optimizing a function on a graph. 
For example, in an A/B/$n$ test with large $n$, we might be interested in finding the version of a web page that has the highest conversion rate. Here, each web page version is a point on a graph and it is connected to versions that are most similar to it (they might vary in only one variable). In this problem, the similarity graph is implicitly given and we can efficiently find the neighbors of each given node. 
Using a continuation method is specially relevant in this problem, given that the feedback is binary, and we might have observed a small number of feedbacks for each version.




In summary, we make the following contributions. First, we show a continuation method for discrete optimization. We show that the method is the best affine approximation of a transformation that is guaranteed to make the function HCF and to have the same global minima as the original function. Second, we apply the method to a class of graph-based nearest neighbor optimization problems, and show the effectiveness of the proposed continuation method. 
Finally, we show that the graph-based nearest neighbor approach is particularly appropriate for big data problems which is a challenging area for many nonparametric methods. This approach is particularly appealing given the success of deep learning methods with large models. 


\section{Other Related work}





The proposed continuation method can be adapted to solve a broad range of problems that involve optimizing a function over a discrete space. A popular approach to discrete optimization is simulated annealing.\todoy{reference} 
Simulated annealing is a Markov Chain Monte Carlo method and considers only the immediate neighbors when making a move. Because of this local property, the algorithm might miss a good direction if the immediate neighbor in that direction has a large value. In contrast, the proposed smoothing approach can infer the global structure of the objective function through longer random walks. In fact, and as we observe in experiments, even random walks of length 1 or 2 can improve the performance as compared to simulated annealing. \todoy{really?!} Other popular methods for solving general optimization problems are evolutionary (genetic) algorithms and the cross-entropy method.  \todoy{references}



We apply the continuation method to graph-based nearest neighbor search. 
We compare the resulting algorithm with two well known nearest neighbor algorithms, namely KDTree and Locality-Sensitive Hashing (LSH) algorithm, on classification problems. The KDTree or k-dimensional tree is a space partitioning binary tree that is used for organizing a set of k-dimensional points in the Euclidean space~\citep{Friedman-Bentley-Finkel-1977,Bentley-1980}. Each leaf node of the KDTree represents a point and each non-leaf node represents a set of points (leafs below that node) and a hyperplane that splits those points into two equal parts. The KDTree is often used to find the nearest neighbor of a query among a set of points. Another popular nearest neighbor search algorithm is the LSH algorithm which uses a hashing technique to map similar input points to the same ``bucket'' with high probability~\citep{Indyk-Motwani-1998}. This hashing mechanism is  different than the conventional hash functions as it tries to maximize the chance of ``collision'' for similar points. 



\section{The Continuation Method}
\label{sec:method}





The basic idea of the continuation method is to deform the objective function such that minimizing the deformed function is easier. To minimize the more complicated function, we instantiate a hill-climbing method at the solution of the simpler function. To be more precise, given a graph $\cG$ with the set of nodes $S$, and a function $f:S \ra \Real$, we are interested in finding a point in $S$ that minimizes $f$. Let $T$ be a sufficiently large number. For $i\in \{1,\dots, T\}$, we define a sequence of functions $\widehat f(., i): S \ra\Real$ such that if $i>j$, then $\widehat f(.,i)$ is more \emph{hill-climbing friendly} than $\widehat f(.,j)$. We let $\widehat f(.,0)=f$. First, starting at a randomly selected point in $S$, we minimize $\widehat f(.,T)$ by a hill-climbing method. Let the minimum be $x_T$. Then, starting at $x_T$, we minimize $\widehat f(.,T-1)$ by a hill-climbing method. We continue this process until we have found $x_0$, which is a local minimum of $\widehat f(.,0) = f$. We construct the smooth function $\widehat f(.,i)$ by performing random walks of length $i$ from each point. 
As we will show, this procedure is an affine approximation of a smoothing operation that transforms $f$ to a HCF function. Given an approximation $\widehat f(.,i)$, any hill-climbing method such as simulated annealing can be used to solve the minimization sub-problem. The procedure is shown in Figures~\ref{alg:opt} and \ref{alg:hillclimbing}.

\begin{figure}
\begin{center}
\framebox{\parbox{8cm}{
\begin{algorithmic}
\STATE \textbf{Input: } Number of rounds $T$. 
\STATE Initialize $x_T = $ random point in $S$. 
\FOR{$t:=T,T-1,\dots, 1$}
\STATE $x_{t-1} = \textsc{Hill-Climbing}(x_t, t)$
\ENDFOR
\STATE Return $x_0$
\end{algorithmic}
}}
\end{center}
\caption{The Optimization Method}
\label{alg:opt}
\end{figure}

\begin{figure}
\begin{center}
\framebox{\parbox{8cm}{
\begin{algorithmic}
\STATE \textbf{Input: } Starting point $x\in S$, length of random walks $s$. 
\WHILE{$x$ not local minima}
\STATE Perform a random walk of length $s$ from the starting state $x$. Let $y$ be the stopping state.
\STATE Let $\widehat f(x, s) = f(y)$ 
\FOR{$x'\in N_x$}
\STATE Perform a random walk of length $s$ from the starting state $x'$. Let $y'$ be the stopping state.
\STATE Let $\widehat f(x', s) = f(y')$.
\ENDFOR
\STATE Update $x = \argmin_{z\in \{x\}\cup N_x} \widehat f(z, s)$  
\ENDWHILE
\STATE Return $x$
\end{algorithmic}
}}
\end{center}
\caption{The \textsc{Hill-Climbing} Subroutine}
\label{alg:hillclimbing}
\end{figure}

Next, we describe the intuition behind the continuation method. In the ideal case and when the function $f$ is HCF, a hill-climbing method finds the global minima. Unfortunately, in many real-world applications, function $f$ is not HCF. Next, we show a transformation of $f$ to a HCF function. Let $N_i\subset S$ be the set of neighbors of node $i\in S$ in $\cG$ and $c\in\Real$ be a non-negative constant. Consider a sequence of functions defined by, $\forall i\in S$, 
\begin{align*}
f(i, 0)&=f(i)\,, \\
f(i, t+1)&=\min( f(i, t),  h(i,t) ) \,,
\end{align*}
where 
\[
h(i,t) = \min_{u\in N_i} f(u,t) + c \;.
\]
Function $f(i, t)$ represents the value of node $i$ at smoothing round $t$. Function $f(i,t)$ is \emph{more hill-climbing friendly} for larger $t$. Let $D$ be the diameter of the graph. We can see that for small enough $c$, the above sequence will eventually produce a HCF function after no more than $D$ steps. Further, the optimum point will not change after this transformation. Let $t$ be large enough such that $f(.,t)$ is HCF. Given that $f(.,t)$ is HCF and $\argmin_{i\in S} f(i) \in \argmin_{i\in S} f(i,t)$, we only need to find the minima of $f(.,t)$. 


Unfortunately, computing  $f(.,t)$ can be computationally expensive. We show that an approximation of $f(.,t)$ however can be computed efficiently. 
\begin{thm}
Let $\cN$ be the operator defining the above sequence; $f(i,t+1) = \cN\{f(.,t)\}(i)$. The best affine approximation of $\cN$ can be computed by performing a random walk of length 1 on the graph.
\end{thm}
\begin{proof}
Before giving details, we define some notation. Let $\cH$ be the class of real-valued functions defined on $S$. We say an operator $\cL$ is linear if and only if $\forall h_1 \in \mathcal{H}, \forall h_2 \in \mathcal{H}, a\in \mathbb{R}, b\in \mathbb{R}; \cL\{ah_1+bh_2\}=a\mathcal{L}\{h_1\}+b\mathcal{L}\{h_2\}$. Consider function $h\in\cH$ and suppose it is a small perturbation of some function $h^*$ in the direction $\phi\in\cH$, that is,
\[
h=h^*+\epsilon \phi
\]
for some small scalar $\epsilon$. Suppose $\mathcal{N}\{h^*+\epsilon \phi\}$ is differentiable in direction $\phi$ so that its first order expansion is 
\begin{align*}
\cN\{h\} &= \cN\{h^*+\epsilon\phi\} \\
&=\cN\{h^*\}+\epsilon\left(\frac{d}{d\epsilon}\cN\{h^*+\epsilon\phi\}\right)_{|\epsilon=0}+o(\epsilon) \;.
\end{align*}
Letting $u_1, u_2, \ldots, u_m$ be the neighbors of node $i$, we have,
\begin{align*}
\cN\{f\}(i)&=\min\{f(i), f(u_1)+c, f(u_2)+c, \ldots\} \\
&\approx \frac{f(i)e^{-\lambda f(i)}+(f(u_1)+c)e^{-\lambda(f(u_1)+c)}+\ldots}{e^{-\lambda f(i)}+e^{-\lambda (f(u_1)+c)}+\ldots}
\end{align*}
for sufficiently large $\lambda>0$. 
Let $f^*$ be the origin in the function space, i.e. $f^*(i)=0$ for any $i$. We can then write 
\begin{align*}
\cN\{f\} &= \cN\{f^*+\epsilon\phi\} \\
&\approx \cN\{f^*\}+\epsilon\left(\frac{d}{d\epsilon}\cN\{f^*+\epsilon\phi\}\right)|_{\epsilon=0}+o(\epsilon)\\
&= \cN\{f^*\}+\epsilon\left(\frac{d}{d\epsilon}\cN\{\epsilon\phi\}\right)|_{\epsilon=0}+o(\epsilon) \;.
\end{align*}
We have
\begin{align*}
&\cN\{\epsilon\phi\}(i)=\\
&\quad\frac{\epsilon\phi(i)e^{-\lambda\epsilon\phi(i)}+(\epsilon\phi(u_1)+c)e^{-\lambda(\epsilon\phi(u_1)+c)}+\ldots}{e^{-\lambda\epsilon\phi(i)}+e^{-\lambda(\epsilon\phi(u_1)+c)}+\ldots}
\end{align*}
By taking derivatives and setting $\epsilon=0$ we will have,
\begin{align*}
&\frac{d}{d\epsilon}\cN\{\epsilon\phi\}(i)|_{\epsilon=0} = \\
&\quad\frac{(\phi(i)+\sum_{j=1}^{m}\phi(u_j)e^{-\lambda c}(1-\lambda c))(1+m e^{-\lambda c})}{(1+m e^{-\lambda c})^2}\\ 
&\qquad+ \frac{m c e^{-\lambda c}\lambda ( \phi(i)+\sum_{j=1}^{n}\phi(u_j)e^{- \lambda c})}{(1+m e^{-\lambda c})^2} \;.
\end{align*}
Thus
\[
\lim_{c\to 0} \frac{d}{d\epsilon}\mathcal{N}\{\epsilon\phi\}(i)|_{\epsilon=0} = \frac{\phi(i)+\sum_{j=1}^{m}\phi(u_j)}{1+m} \;.
\]
Therefore,
\begin{align*}
\mathcal{N}\{f\}(i)&\approx\epsilon \frac{\phi(i)+\sum_{j=1}^{m}\phi(u_j)}{1+m}\\
&=\frac{f(i)+\sum_{j=1}^{m}f(u_j)}{1+m} \;.
\end{align*}
Let $A\in\Real^{n\times n}$ be the adjacency matrix underlying graph $\cG\in\Real^{n\times n}$ and $D$ be the diagonal matrix representing the degrees of the nodes in the graph. Let $P\in \Real^{n\times n}$ be the stochastic transition matrix defined by 
\[
P=(D+I)^{-1}(A+I) \;. 
\]
We observe that
\[
(Pf)(i) = \frac{f(i)+\sum_{j=1}^{n}f(u_j)}{1+n} \;.
\]
Therefore, if we approximate $\cN$ by its first order expansion, 
\begin{align*}
\widehat f(., 0)&=f\,, \\
\widehat f(., t+1)&= P \widehat f(.,t) \;.
\end{align*}
Applying $P$ to the function $f$ is equivalent to doing a random walk on the nodes of the graph according to $P$. Therefore, $\widehat f(., t+1)$ can be obtained by simulating a random walk of length $t$ and returning value of $f$ at the stopping point.  
\end{proof}

We summarize the optimization method as follows. Let $T$ be a sufficiently large number. First, we find the local minima of $\widehat f(., T)$. This is obtained by running random walks of length $T$. Then, starting at this local minima, we find the local minima of $\widehat f(., T-1)$. We continue this process until we find the local minima of $\widehat f(., 0)$, which is returned as the approximate minimizer of $f$.

\subsection{Graph-Based Nearest Neighbor Classification}
\label{sec:graph-nn}


The basic idea is to construct a data structure that allows for fast test phase nearest neighbor computation. First, we explain the construction of the data structure, and then we will discuss the search problem on this data structure. 
We construct a proximity graph over training data. Let $N$ be a positive integer. Let $\cG$ be a proximity graph constructed on $S$ in an offline phase, i.e. set $S$ is the set of nodes of $\cG$, and each point in $S$ is connected to its $N$ nearest neighbors with respect to some distance metric. In our experiments, we use the Euclidean metric. This data structure is particularly well suited to big-data problems. The intuition is that, in big-data problems, points will have close neighbors, and so even a simple metric such as Euclidean metric should perform well. 

\todoy{discuss construction of this data structure, complexity, efficient methods, ...}

Let $f = d(.,y)$ be the optimization objective. Given the graph $\cG$, the problem is reduced to minimizing function $f$ over a graph $\cG$. 
The graph based nearest neighbor search has been studied by~\citet{Arya-Mount-1993, Brito-1997, Eppstein-1997, Miller-Teng-1997, Plaku-Kavraki-2007, Chen-Fang-Saad-2010, Connor-Kumar-2010, Dong-Charikar-Li-2011, Hajebi-2011, Wang-Wang-2012}. \citet{Laarhoven-2017} show that the theoretical time and space complexity of graph-based nearest neighbor search can be competitive with hash-based methods in certain regimes. 



We can use a hill-climbing method such as simulated annealing to optimize $f$ over $\cG$. Efficiency of a hill-climbing method depends on the shape of $f$ and connectivity of the graph. The problem is easier when $f$ is HCF enough with respect to the geometry implied by the graph. If $f$ is very irregular, we might need many random restarts to achieve a reasonable solution. Figure~\ref{fig:sparse-graph} shows a nearest neighbor search problem given query point $q$. The nearest neighbor in the graph is $y$. If we start the hill-climbing procedure from $z_0$, we will end up at $z_2$ which is a local minima. Figure~\ref{fig:dense-graph} shows the same problem with a bigger training data. As the graph is more dense, the hill-climbing is more likely to end up in a point that is closer to the query point. Thus, we will need a smaller number of restarts to achieve a certain level of accuracy. 


\begin{figure}[!h]
    \centering
    \subfloat[]{{\includegraphics[width=0.35\textwidth]{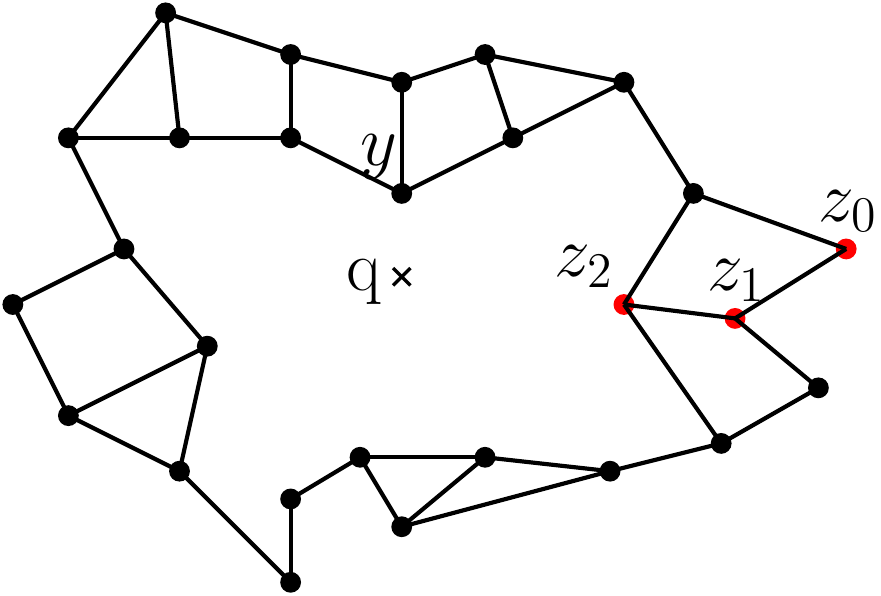} 	}}
    \vspace{0cm}
    \caption{The hill-climbing procedure gets stuck in a local minima in a sparse graph.} \label{fig:sparse-graph}%
\end{figure}

\begin{figure}[!h]
    \centering
    \subfloat[]{{\includegraphics[width=0.35\textwidth]{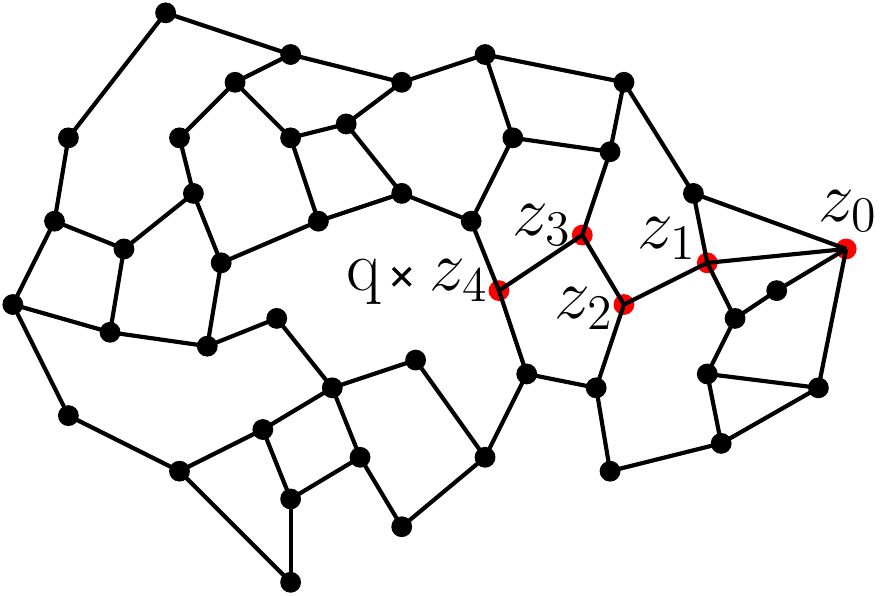} 	}}
    \vspace{0cm}
    \caption{The hill-climbing procedure finds the global minima given a larger training data.} \label{fig:dense-graph}%
\end{figure}





We apply the continuation method to the problem of minimizing $f$ over graph $\cG$. We call the resulting algorithm SGNN for Smoothed Graph-based Nearest Neighbor search. The algorithm is shown in Figures~\ref{alg:opt-rr} and~\ref{alg:simulatedannealing}. This algorithm has several differences compared to the basic algorithm in the previous section. 
First, the SGNN uses a simulated annealing procedure instead of the hill-climbing procedure. Second, instead of stopping the hill-climbing procedure in a local minima, the SGNN continues for a fixed number of iterations. In our experiments, we run the simulated annealing procedure for $\log n$ rounds, where $n$ is the size of the training set. See Figure~\ref{alg:simulatedannealing} for a pseudo-code. Finally, the SGNN runs the simulated annealing procedure several times and returns the best outcome of these runs. The resulting algorithm with random restarts is shown in Figure~\ref{alg:opt-rr}. 
We will show the performance of the proposed method in two classification problems in the experiments section. 
The above algorithm returns an approximate nearest neighbor point. To find $K$ nearest neighbors for $K>1$, we simply return the best $K$ elements in the last line in Figures~\ref{alg:opt-rr}.

Choice of $N$ impacts the prediction accuracy and computation complexity; smaller $N$ means lighter training phase computation, and heavier test phase computation (as we need more random restarts to achieve a certain prediction accuracy). Having a very large $N$ will also make the test phase computation heavy.

\begin{figure}
\begin{center}
\framebox{\parbox{8cm}{
\begin{algorithmic}
\STATE \textbf{Input: } Number of random restarts $I$, number of hill-climbing steps $J$, length of random walks $T$.
\STATE Initialize set $U=\{\}$
\FOR{$i:=1,\dots,I$}
\STATE Initialize $x = $ random point in $S$. 
\STATE $x' = \textsc{Smoothed-Simulated-Annealing}(x, T, J)$
\STATE $U=U\cup \{x'\}$
\ENDFOR
\STATE Return the best element in $U$
\end{algorithmic}
}}
\end{center}
\caption{The Optimization Method with Random Restarts}
\label{alg:opt-rr}
\end{figure}

\begin{figure}
\begin{center}
\framebox{\parbox{8cm}{
\begin{algorithmic}
\STATE \textbf{Input: } Starting point $x\in S$, number of hill-climbing steps $J$, length of random walks $T$. 
\FOR{$j:=1,\dots,J$}
\STATE Perform a random walk of length $T$ from $x$. Let $y$ be the stopping state.
\STATE Let $\widehat f(x, s) = f(y)$ 
\STATE Let $u$ be a neighbor of $x$ chosen uniformly at random.
\STATE Perform a random walk of length $T$ from $u$. Let $v$ be the stopping state.
\STATE Let $\widehat f(u, s) = f(v)$.
\IF{$f(v)\le f(y)$}
\STATE Update $x = u$  
\ELSE
\STATE Temperature $\tau=1-j/J$
\STATE With probability $e^{(f(y)-f(v))/\tau}$, update $x=u$
\ENDIF
\ENDFOR
\STATE Return $x$
\end{algorithmic}
}}
\end{center}
\caption{The \textsc{Smoothed-Simulated-Annealing} Subroutine}
\label{alg:simulatedannealing}
\end{figure}

\section{Experimental Results}

We compared the SGNN method with the state-of-the-art nearest neighbor search methods in two image classification problems. We use $K=50$ approximate nearest neighbors to predict a class for each given query. We used the MNIST and COIL-100 datasets, that are standard datasets for image classification. The MNIST dataset is a black and white image dataset, consisting of 60000 training images and 10000 test images in 10 classes. Each image is $28\times 28$ pixels. The COIL-100 dataset is a colored image dataset, consisting of 100 objects, and 72 images of each object at every 5x angle. Each image is $128\times 128$ pixels, We use 80\% of images for training and 20\% of images for testing. 

We construct a directed graph $\cG$ by connecting each node to its 30 closest nodes in Euclidean distance. For smoothing, we use random walks of length $T=1$. (We will also show results with $T=2$.) For the SGNN algorithm, 
the number of hill-climbing steps is $J=\log(\text{training size})$ in each restart. 
We pick the number of restarts so that all methods have similar prediction accuracy. The SGNN method with $T=1$ is denoted by SGNN(1), and SGNN with $T=0$, i.e. pure simulated annealing on the graph, is denoted by SGNN(0). 

For LSH and KDTree algorithms, we use the implemented methods in the scikit- learn library with the following parameters. For LSH, we use 
LSHForest with min hash match=4, \#candidates=50, \#estimators=50, \#neighbors=50, radius=1.0, radius cutoff ratio=0.9. For KDTree, we use leaf size=1 and $K$=50, meaning that indices of 50 closest neighbors are returned. The KDTree method always significantly outperforms LSH, so we compare only with KDTree. 


\todoy{have we tried smoothing with longer random walks, say 3? other graphs? other random restarts?}
\todoy{say that we are doing better than simulated annealing...}

Figure~\ref{fig:mnist} (a-d) shows the accuracy of different methods on different portions of MNIST dataset (25\%, 50\%, 75\%, and 100\% of training data). 
Using the exact nearest neighbor search, we get the following prediction accuracy results (the error bands are 95\% bootstrapped confidence intervals): with full data, accuracy is $0.955\pm 0.01$; with 3/4 of data, accuracy is $0.951\pm 0.01$; with 1/2 of data, accuracy is $0.943\pm 0.01$; and with 1/4 of data, accuracy is $0.932\pm 0.01$. As the size of training set increases, the prediction accuracy of all methods improve. Figure~\ref{fig:mnist} (e-h) shows that the test phase runtime of the SGNN method has a more modest growth for larger datasets. In contrast, KDTree becomes much slower for larger training datasets. When using all training data, the SGNN method has roughly the same accuracy, but it has less than 20\% of the test phase runtime of KDTree. 
Figure~\ref{fig:coil} (a-d) shows the accuracy of different methods on different portions of COIL-100 dataset. As the size of training set increases, the prediction accuracy of all methods improve. Figure~\ref{fig:coil} (e-h) shows that the test phase runtime of the SGNN method has a more modest growth for larger datasets. In contrast, KDTree becomes much slower for larger training datasets. When using all training data, the proposed method has roughly the same accuracy, while having less than 50\% of the test phase runtime of KDTree. 

These results show the advantages of using graph-based nearest neighbor algorithms; as the size of training set increases, the proposed method is much faster than KDTree. The proposed continuation method also outperforms simulated annealing, denoted by SGNN(0), in these experiments. The simulated annealing is more likely to get stuck in local minima, and thus it requires more random restarts to achieve an accuracy that is comparable with the accuracy of the continuation method. This explains the difference in the runtimes of SGNN(0) and SGNN(1).  

\begin{figure}
    \centering
    \subfloat[]{{\includegraphics[width=0.3\textwidth]{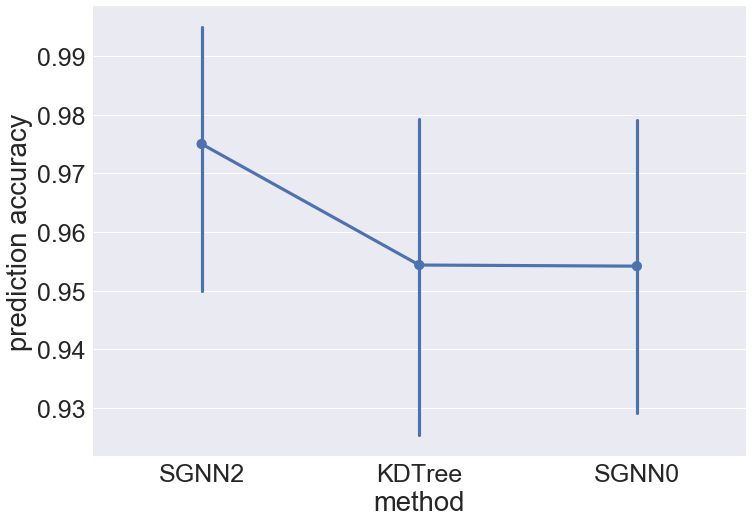} 	}}\\
	\subfloat[]{{\includegraphics[width=0.3\textwidth]{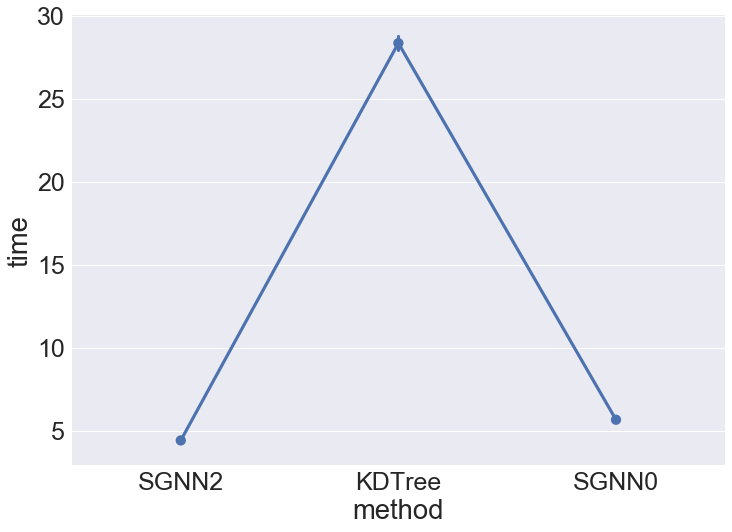} 	}} 	 \\
    \subfloat[]{{\includegraphics[width=0.3\textwidth]{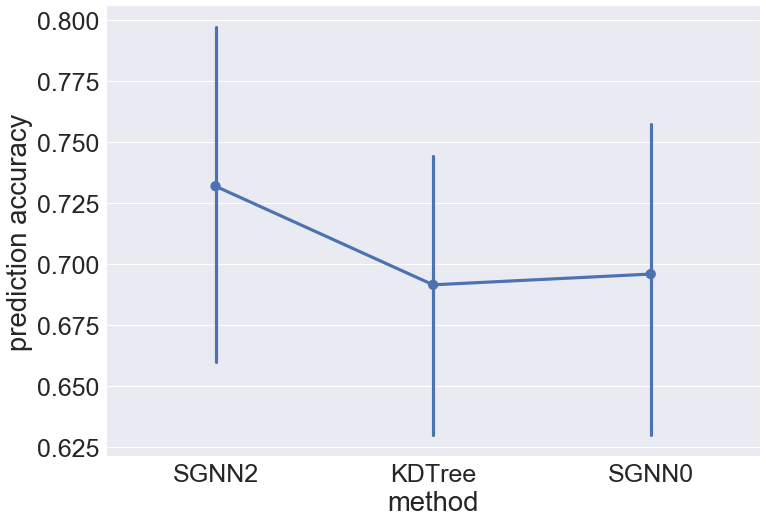} 	}} \\
   \subfloat[]{{\includegraphics[width=0.3\textwidth]{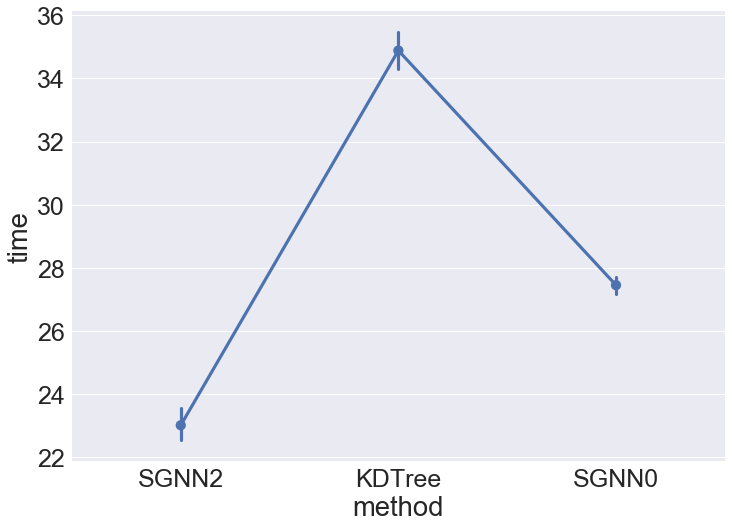} 	}} 
    \vspace{0cm}
    \caption{Prediction accuracy and running time of the SGNN method with random walks of length two (a) Accuracy on MNIST dataset using 100\% of training data. (b) Running time on MNIST dataset using 100\% of training data. (c) Accuracy on COIL-100 dataset using 100\% of training data. (d) Running time on COIL-100 dataset using 100\% of training data.} \label{fig:longwalk}%
\end{figure}


\begin{figure*}
    \centering
    \subfloat[]{{\includegraphics[width=0.25\textwidth]{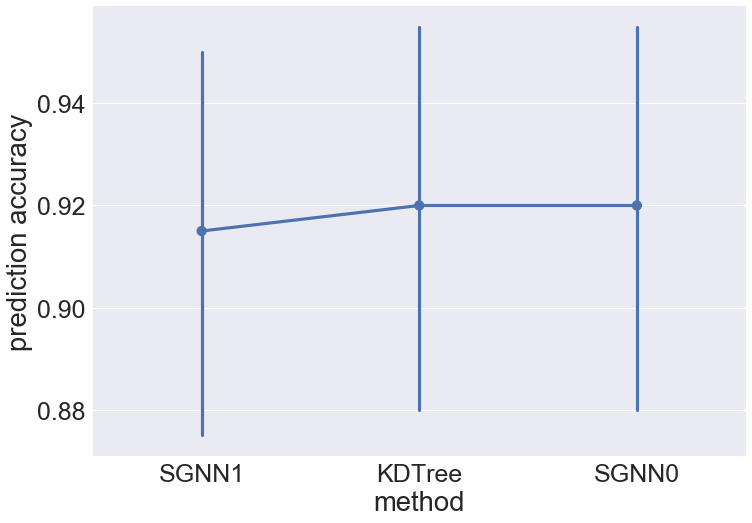} 	}}
	\subfloat[]{{\includegraphics[width=0.25\textwidth]{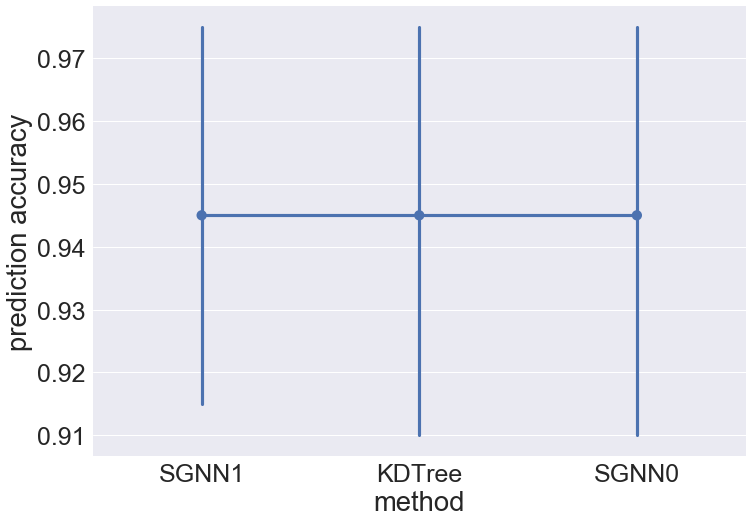} 	}} 	 
    \subfloat[]{{\includegraphics[width=0.25\textwidth]{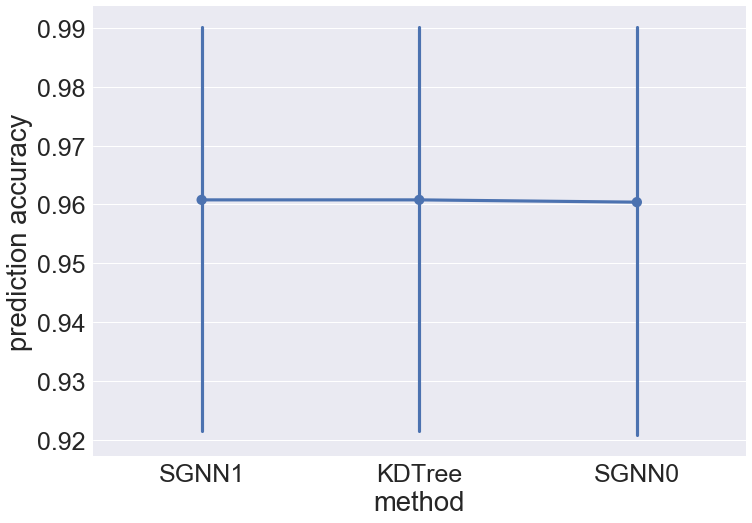} 	}} 
   \subfloat[]{{\includegraphics[width=0.25\textwidth]{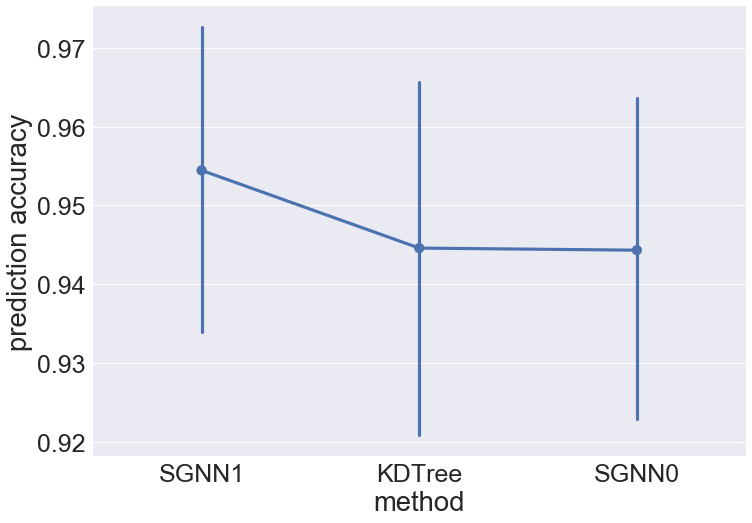} 	}} \\
    \subfloat[]{{\includegraphics[width=0.25\textwidth]{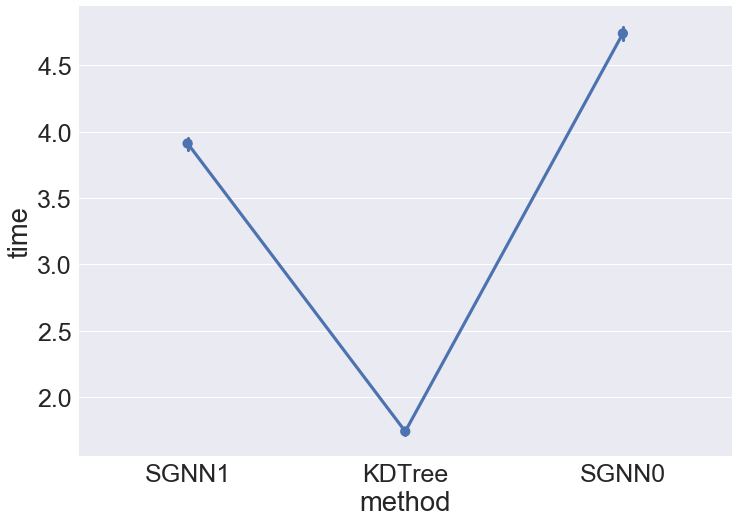} 	}}
	\subfloat[]{{\includegraphics[width=0.25\textwidth]{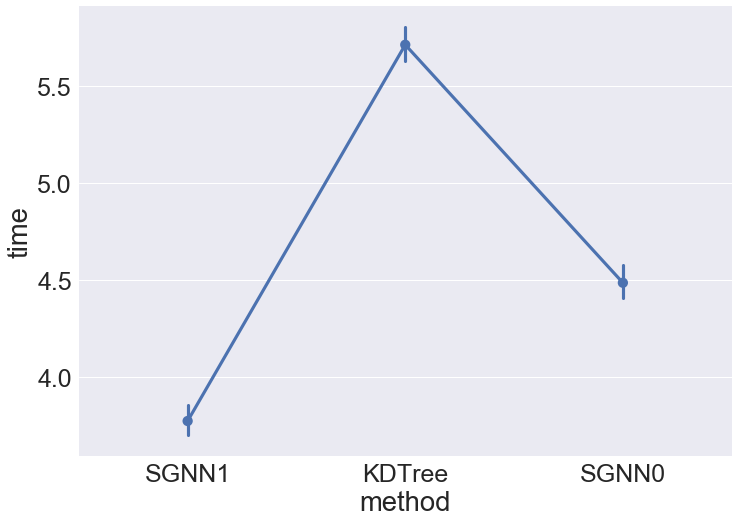} 	}} 	 
    \subfloat[]{{\includegraphics[width=0.25\textwidth]{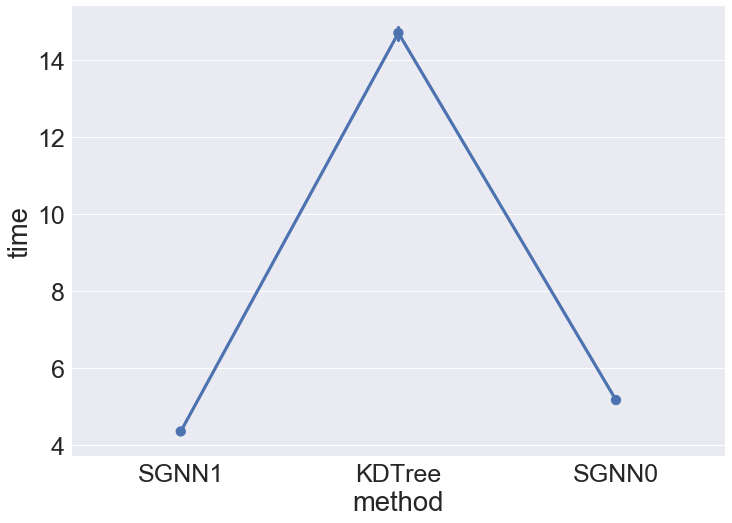} 	}} 
   \subfloat[]{{\includegraphics[width=0.25\textwidth]{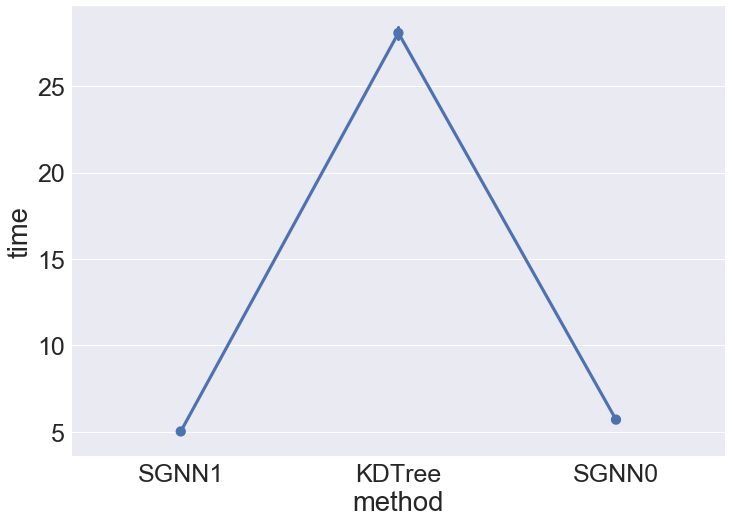} 	}}    
    \vspace{0cm}
    \caption{Prediction accuracy and running time of different methods on MNIST dataset as the size of training set increases. (a,e) Using 25\% of training data. (b,f) Using 50\% of training data. (c,g) Using 75\% of training data. (d,h) Using 100\% of training data.} \label{fig:mnist}%
\end{figure*}

\begin{figure*}
    \centering
    \subfloat[]{{\includegraphics[width=0.25\textwidth]{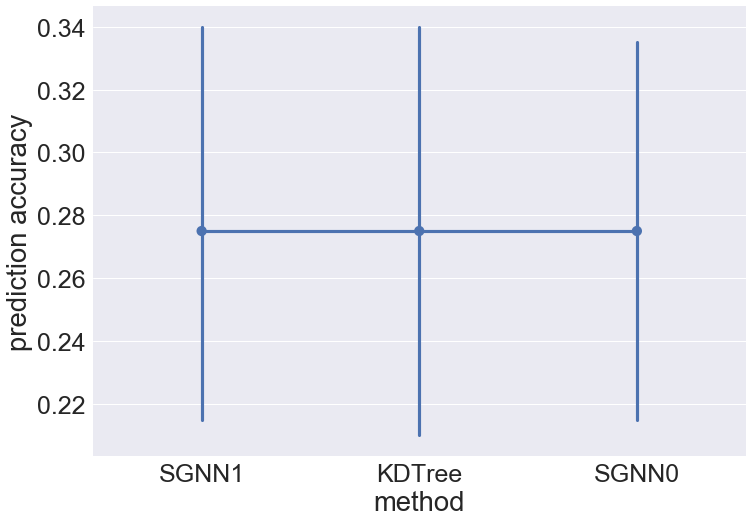} 	}}
	\subfloat[]{{\includegraphics[width=0.25\textwidth]{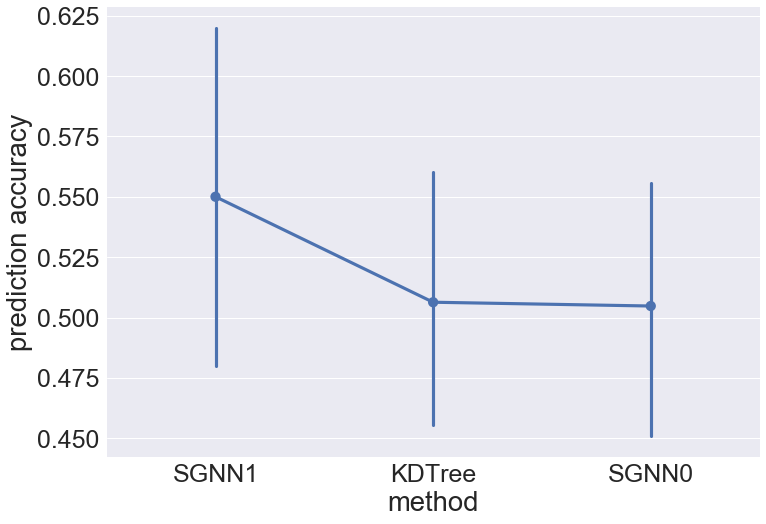} 	}} 	 
    \subfloat[]{{\includegraphics[width=0.25\textwidth]{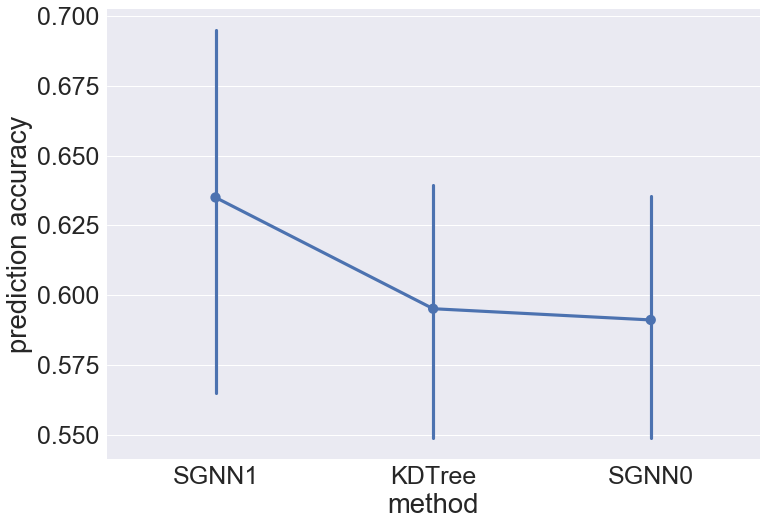} 	}} 
   \subfloat[]{{\includegraphics[width=0.25\textwidth]{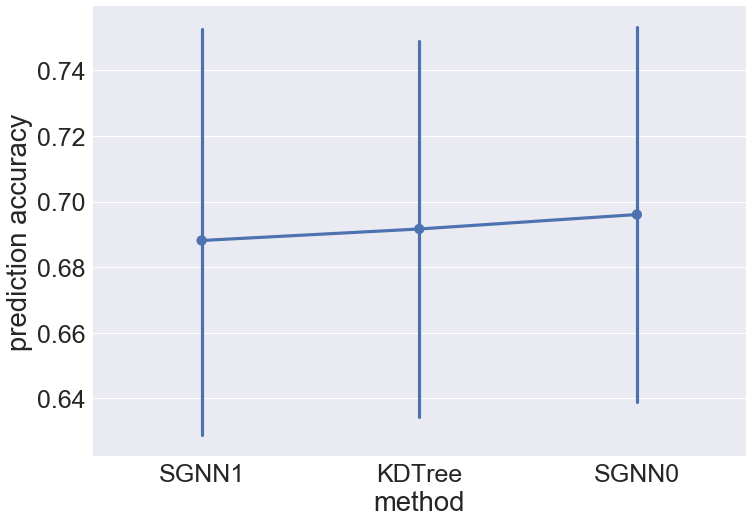} 	}} \\
    \subfloat[]{{\includegraphics[width=0.25\textwidth]{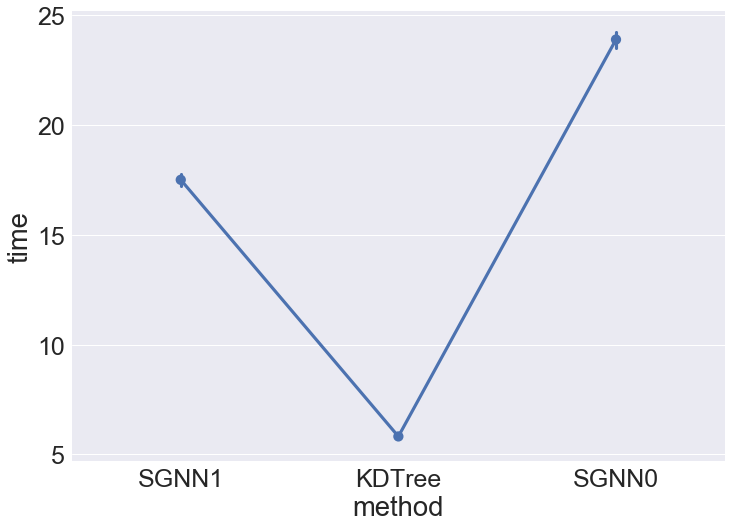} 	}}
	\subfloat[]{{\includegraphics[width=0.25\textwidth]{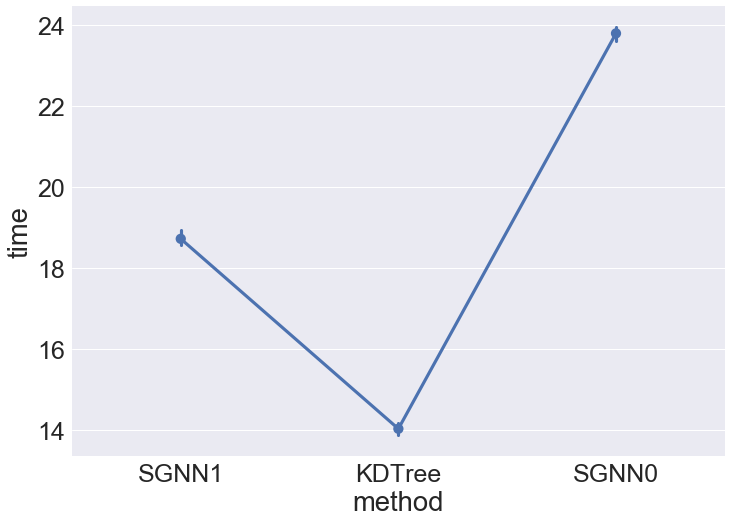} 	}} 	 
    \subfloat[]{{\includegraphics[width=0.25\textwidth]{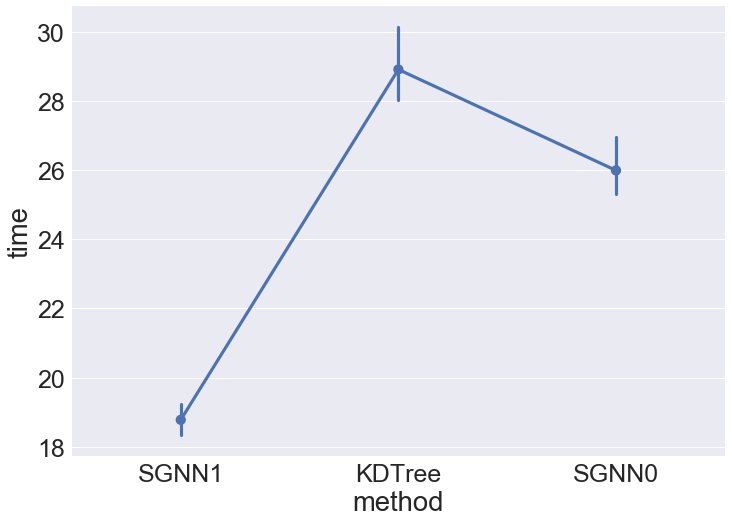} 	}} 
   \subfloat[]{{\includegraphics[width=0.25\textwidth]{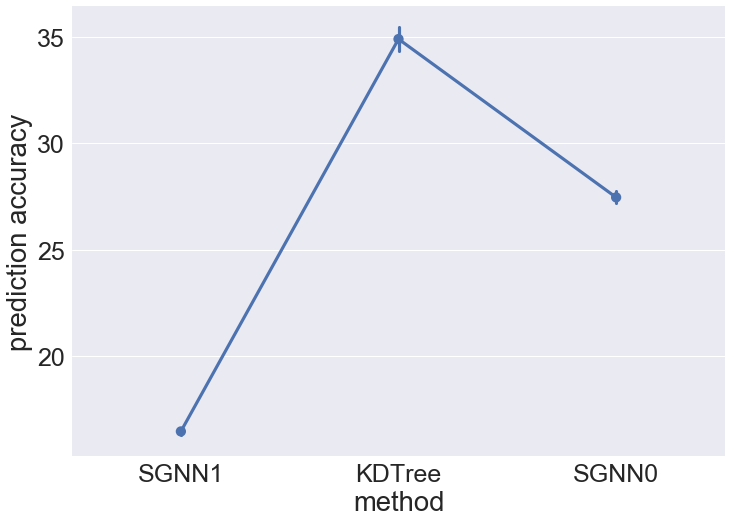} 	}} 
    \vspace{0cm}
    \caption{Prediction accuracy and running time of different methods on COIL-100 dataset as the size of training set increases. (a,e) Using 25\% of training data. (b,f) Using 50\% of training data. (c,g) Using 75\% of training data. (d,h) Using 100\% of training data.} \label{fig:coil}%
\end{figure*}

Next, we study how the performance of SGNN changes with the length of random walks. We choose $T=2$ and compare different methods on the same datasets. The results are shown in Figure~\ref{fig:longwalk}. The SGNN(2) method outperforms the competitors. Interestingly, SGNN(2) also outperforms the exact nearest neighbor algorithm on the MNIST dataset. This result might appear counter-intuitive, but we explain the result as follows. Given that we use a simple metric (Euclidean distance), the exact $K$-nearest neighbors are not necessarily appropriate candidates for making a prediction; Although the exact nearest neighbor algorithm finds the global minima, the neighbors of the global minima on the graph might have large values. On the other hand, the SGNN(2) method finds points that have small values and also have neighbors with small values. This stability acts as an implicit regularization in the SGNN(2) algorithm, leading to an improved performance.

\section{Conclusions}

We showed a continuation method for discrete optimization problems. The method is the best affine approximation of a deformation that is guaranteed to produce a HCF approximation with the same global minima. We applied the continuation method to a graph-based nearest neighbor search, and showed improved performance on two image classification domains. The nearest neighbor algorithm has a number of appealing features. In particular, the runtime in the test phase grows modestly with the size of training set. 




\clearpage

\bibliography{biblio}

\end{document}